\documentclass[11pt,letterpaper,onecolumn]{article}
\usepackage[T1]{fontenc}
\usepackage[utf8]{inputenc}
\usepackage{lmodern}
\usepackage[left=1.25in,right=1.25in,top=1in,bottom=1in]{geometry}
\usepackage{amsmath,amssymb,amsthm,mathtools,bm}
\usepackage{microtype,booktabs,array,tabularx,enumitem,needspace}
\usepackage[numbers,sort&compress]{natbib}
\usepackage[hidelinks]{hyperref}
\hypersetup{pdftitle={EntroPrefill: Renyi-Guided Context Pruning with Conditional Stability Guarantees},pdfauthor={Inbasekaran S.;}}
\newtheorem{theorem}{Theorem}[section]
\newtheorem{lemma}[theorem]{Lemma}
\newtheorem{proposition}[theorem]{Proposition}
\newtheorem{corollary}[theorem]{Corollary}
\theoremstyle{definition}

\theoremstyle{remark}

\newcommand{\E}{\mathbb E}
\newcommand{\Prb}{\mathbb P}

\newcommand{\ind}{\mathbf 1}
\newcommand{\norm}[1]{\left\lVert #1\right\rVert}

\newcommand{\softmax}{\operatorname{softmax}}
\newcommand{\diag}{\operatorname{diag}}
\newcommand{\UOR}{\operatorname{UOR}}
\newcommand{\supp}{\operatorname{supp}}

\newcommand{\pos}[1]{\left[#1\right]_+}
\allowdisplaybreaks[1]
\title{\bfseries\Large EntroPrefill: R\'enyi-Guided Context Pruning\\
with Conditional Stability Guarantees\\
for Retrieval-Augmented Generation}
\author{Inbasekaran S.\textsuperscript{1}\\[3pt]
\small\textsuperscript{1}SRM Institute of Science and Technology, India\\}
\date{}

\begin{document}
\maketitle
\begin{abstract}
Mid-prefill pruning can reduce the sequence processed by deeper transformer layers, but attention concentration alone does not certify that discarded context is dispensable. We formulate EntroPrefill as a R\'enyi-guided proposal mechanism coupled to explicit constraints on discarded attention mass. Sink-isolated, regularized head pooling respects grouped-query attention while exposing a quantitative trade-off between specialization and worst-head coverage. We derive a mixture-to-head deletion envelope, a computable upper bound on feasible token removal, and a finite-sample observer guarantee that remains valid when the pruning layer is selected adaptively. We then establish a conditional transformer perturbation bound with explicit sufficient Lipschitz constants and a first-token decision-margin corollary. A counterexample shows why shallow observations alone cannot imply an unconditional future-output guarantee. The systems analysis distinguishes query-head unions, physical page allocation, and KV-transfer payload, and gives an arithmetic break-even condition for pruning. These results define a theoretical design and its limits; they do not assert measured acceleration or preserved task accuracy. Experiments are reserved for subsequent validation of the assumptions, approximation tightness, and end-to-end resource trade-offs.
\end{abstract}

\section{Introduction \& Related Work}

Retrieval-augmented generation augments a query with document passages that differ in their relevance and computational value. A passage can contain a directly useful fact, an intermediate reasoning dependency, or a distractor. Processing every passage through every transformer layer incurs work even when some content is unnecessary for the response. Removing passages during prefill therefore offers a natural optimization opportunity, provided the selection rule distinguishes a computational saving from an unsupported claim of semantic preservation.

Attention-based selection has substantial prior art. LLMLingua compresses prompts before target-model inference, while AttentionRAG uses query-focused attention for context pruning \citep{jiang2023,fang2025}. LazyLLM dynamically selects prompt tokens and supports their reconsideration during generation; SlimInfer prunes intermediate hidden states; ASL adaptively selects a pruning layer from attention-rank behavior \citep{fu2024,long2025,taniguchi2026}. These works preclude claiming novelty for attention-guided compression, intermediate-layer pruning, or adaptive layer selection in isolation. EntroPrefill instead studies how an entropy-guided proposal can be coupled to explicit deletion constraints, adaptive statistical auditing, and conditional output bounds.

Decode-cache policies address a related resource dimension. H2O balances recent and high-importance tokens; SnapKV derives important positions from a prompt observation window; Ada-KV studies adaptive cache-budget allocation and bounds attention-output loss \citep{zhang2023,li2024,feng2024}. The local relationship between discarded attention mass and output perturbation is thus not introduced here as an unprecedented principle. Our analysis develops its consequences for a global chunk-removal decision that changes the support processed by subsequent layers, where unobserved attention and representation drift become essential assumptions.

Grouped-query attention (GQA) shares keys and values among query heads \citep{ainslie2023}. Page-EntroKV proposes sink-isolated entropy weighting and page-level selection within these groups \citep{pageentro}. Attention sinks motivate conditioning entropy away from designated initial positions, but concentration after conditioning still does not identify factual correctness \citep{xiao2024,wu2024}. At the execution level, FlashAttention-2, PagedAttention, DistServe, and Sarathi-Serve optimize different parts of the computation and resource schedule \citep{dao2023,kwon2023,zhong2024,agrawal2024}. Their benefits do not transfer automatically to a newly combined pruning system.

The mathematical contribution is a unified conditional analysis. First, a mixture envelope converts pooled discarded mass into a bound for an individual observed attention row and quantifies when pooling loses useful coverage. Second, a constrained removal formulation admits a dual upper bound, separating feasible pruning from unattainable budget targets. Third, an independent observer audit controls selected-layer error without treating adaptive stopping as a fixed test. Fourth, a blockwise stability theorem states the additional assumptions needed to propagate deletion through a pre-normalized transformer. Elementary probability, convex-mixture identities, and norm inequalities are identified as established ingredients; the proposed contribution is their coupling to one fully specified pruning procedure.

\section{Problem Formulation}

\subsection{Prompt support and GQA attention}

Let the prompt contain protected tokens $\mathcal P$ and disjoint candidate document chunks $c_1,\ldots,c_k$. Instructions, the complete user query, required delimiters, and configured initial sink tokens are protected. If protection intersects a chunk, the entire chunk is protected before indexing the remaining candidates. Write $n_i=|c_i|$, $\mathcal D=\bigcup_i c_i$, $N=\sum_i n_i$, and $T=|\mathcal P|+N$. All candidate documents precede the user-query observer positions. If $k=0$, the method performs ordinary inference.

Consider a decoder-only transformer with $L$ blocks, model width $d_m$, $H_Q$ query heads, $G$ KV groups, and head dimension $d$. Let $r=H_Q/G$ be an integer, and let $\mathcal H_g$ contain the query heads sharing group $g$. At layer $l$, query position $u$, and head $h$, attention is
\begin{equation}
 A^{l,h}_{u,j}=\frac{\exp z^{l,h}_{u,j}}{\sum_{v\le u}\exp z^{l,h}_{u,v}},
 \qquad z^{l,h}_{u,j}=\frac{q^{l,h}_u\cdot k^{l,g(h)}_j}{\sqrt d},\quad j\le u,
 \label{eq:attention}
\end{equation}
and is zero on masked positions. Original position IDs are used throughout. Every document token is visible to every user-query observer; no causal-visibility correction favoring later documents follows from this layout.

Choose a fixed collection $\mathcal Q_p$ of $W_p\ge1$ proposal observers from the user query. They may be a deterministic suffix or a random sample drawn independently of the audit described in Section~\ref{sec:audit}. Define
\begin{equation}
 a^{l,h}_j=\frac1{W_p}\sum_{u\in\mathcal Q_p}A^{l,h}_{u,j}.
 \label{eq:head-average}
\end{equation}
The averaging is over query rows, not over a truncated key window. All causally visible key positions remain in the distribution.

\subsection{Decisions, guarantees, and norms}

A removal set $E\subseteq\mathcal D$ is a union of complete candidate chunks. The retained support is $R=[T]\setminus E$, containing $\mathcal P$, with length $T'=T-|E|$. At the selected layer $l^*$, the method keeps $X^{l^*}|_R$ and the shallow-layer KV entries indexed by $R$, then processes only $R$ through the remaining layers. It does not restart the model on a shortened text prompt.

We distinguish an exact guarantee on specified proposal rows, a probabilistic guarantee for a specified observer distribution, and a conditional guarantee for the full-model output. These statements concern different objects. For a retained-state matrix $X$, use the row-maximum norm $\norm{X}_{\max,2}=\max_u\norm{X_u}_2$. Vector and matrix norms without another subscript are Euclidean and spectral, respectively. Attention measures use natural logarithms. The group-floor parameter is $\zeta\in(0,1]$, the entropy temperature is $\tau>0$, and numerical residual thresholds are positive.

\section{Sink-Isolated R\'enyi Pooling}

\subsection{Regularized head weighting}

Let $\mathcal S\subseteq\mathcal P$ be the sink set. For each head, set $z_h=\sum_{j\notin\mathcal S}a^{l,h}_j$. A head is entropy-eligible when $z_h\ge z_{\min}>0$. For eligible heads, define the conditional distribution on $j\notin\mathcal S$ by
\begin{equation}
 \widetilde a^h_j=\frac{a^{l,h}_j}{z_h},\qquad
 C_h=\sum_{j\notin\mathcal S}(\widetilde a^h_j)^2,
 \qquad H^{\rm ns}_{2,h}=-\ln C_h.
 \label{eq:collision}
\end{equation}
For a group with at least one eligible head, assign eligible heads
\begin{equation}
 v_h=\frac{C_h^{1/\tau}}{\sum_{a\in\mathcal H_g:\,z_a\ge z_{\min}}C_a^{1/\tau}},
 \qquad w_h=(1-\zeta)v_h+\frac{\zeta}{r},
 \label{eq:weights}
\end{equation}
with $v_h=0$ for ineligible heads. If all heads in a group are ineligible, set $v_h=1/r$. This fallback preserves normalization; if every group falls back, the algorithm bypasses pruning. Collision sums and normalizations are accumulated in FP32 or greater precision.

The weights are exponential in negative entropy, since $C_h^{1/\tau}=\exp(-H^{\rm ns}_{2,h}/\tau)$. They are not reciprocals of entropy. The positive floor prevents a numerically or statistically suppressed head from receiving zero mixture coverage. Pool the original attention distributions:
\begin{equation}
 \mu^l_j=\frac1G\sum_{g=1}^{G}\sum_{h\in\mathcal H_g}w_h a^{l,h}_j
 =\sum_{h=1}^{H_Q}\sum_{u\in\mathcal Q_p}\omega_{h,u}A^{l,h}_{u,j},
 \qquad \omega_{h,u}=\frac{w_h}{GW_p}.
 \label{eq:mixture}
\end{equation}
Thus $\mu^l$ is a probability distribution and $\omega_{h,u}\ge\zeta/(H_QW_p)$. Entropy controls the proposal; the floor explicitly limits how strongly it may suppress a head.

\begin{proposition}[Entropy separation and its floor]
\label{prop:entropy}
Suppose a group contains an eligible head $h_*$ with entropy $H_*$ and $s-1$ other eligible heads with entropy at least $H_*+\Delta$, where $s\le r$ and $\Delta\ge0$. Then
\begin{equation}
 w_{h_*}\ge\frac{\zeta}{r}+\frac{1-\zeta}{1+(s-1)e^{-\Delta/\tau}}.
 \label{eq:gap}
\end{equation}
If an eligible sink-dominated head has a uniform residual on $n$ non-sink positions and an eligible sibling has collision at least $c_0>0$, then
\begin{equation}
 w_{\rm sink}\le\frac{\zeta}{r}+
 (1-\zeta)\frac{n^{-1/\tau}}{n^{-1/\tau}+c_0^{1/\tau}}.
 \label{eq:sink}
\end{equation}
\end{proposition}
\begin{proof}
Divide the denominator of $v_{h_*}$ by $\exp(-H_*/\tau)$. Each competing term is at most $\exp(-\Delta/\tau)$, yielding~\eqref{eq:gap}. A uniform residual has collision $1/n$, whereas the specified sibling contributes at least $c_0^{1/\tau}$ to the denominator. Substitution into~\eqref{eq:weights} gives~\eqref{eq:sink}.
\end{proof}

For fixed positive $\zeta$, the sink weight approaches the floor $\zeta/r$, not zero. Without the floor, its unnormalized factor decays polynomially in $n$. If all siblings have identical diffuse residuals, normalized weights remain uniform. If a sink head's residual is itself sharply concentrated, the uniform-residual premise fails. These qualifications prevent the entropy statistic from being interpreted as a universal retrieval-head detector. The standard R\'enyi family and its monotonicity are discussed by van Erven and Harremo\"es \citep{vanerven2014}; order two is chosen here for its explicit second-moment representation, not for universal statistical optimality.

\subsection{Null-calibrated chunk priorities}

Let $Z_l=\mu^l(\mathcal D)$. When $Z_l<Z_{\min}$, bypass pruning. Otherwise, define document-conditioned probabilities $p_j=\mu^l_j/Z_l$, chunk mass $M_i=\sum_{j\in c_i}p_j$, and null mass $\pi_i=n_i/N$. For smoothing parameter $\varepsilon_s>0$, set
\begin{equation}
 \widehat p_{i,j}=\frac{p_j+\varepsilon_s\pi_i/n_i}{M_i+\varepsilon_s\pi_i},
 \qquad \ell_i=\frac{M_i+\varepsilon_s\pi_i}{(1+\varepsilon_s)\pi_i}.
 \label{eq:chunk-smoothing}
\end{equation}
For $n_i>1$, define
\begin{equation}
 \kappa_i=\frac{n_i\sum_{j\in c_i}\widehat p_{i,j}^{2}-1}{n_i-1},
 \qquad S_i=\ln\ell_i+\lambda\kappa_i,\quad \lambda\ge0,
 \label{eq:score}
\end{equation}
and set $\kappa_i=0$ for singleton chunks. The score determines a proposal ordering; it does not replace discarded-mass constraints.

\begin{proposition}[Calibration and bounded concentration preference]
\label{prop:score}
For every chunk, $0\le\kappa_i\le1$. Uniform attention on $\mathcal D$ gives $S_i=0$ for every chunk. For any two chunks,
\begin{equation}
 S_i-S_j\ge\ln(\ell_i/\ell_j)-\lambda.
 \label{eq:dominance}
\end{equation}
In particular, $\ell_i/\ell_j>e^\lambda$ implies $S_i>S_j$, independently of their within-chunk concentration.
\end{proposition}
\begin{proof}
For a probability vector on $n$ positions, Cauchy--Schwarz gives $\sum p_j^2\ge1/n$, while $p_j^2\le p_j$ gives $\sum p_j^2\le1$. Equation~\eqref{eq:score} yields the range. Under the uniform null, $M_i=\pi_i$, $\widehat p_{i,j}=1/n_i$, and $\ell_i=1$. Finally, $\kappa_i-\kappa_j\ge-1$ proves~\eqref{eq:dominance}.
\end{proof}

The last inequality is a precise, limited safeguard for diffuse evidence: a sufficiently large relative-mass advantage cannot be overturned by concentration alone. It does not assert that all diffuse relevant chunks possess that advantage. Lowering the R\'enyi order would not supply this property, because $H_1(P)\ge H_\alpha(P)\ge H_2(P)$ for $1\le\alpha\le2$ implies the reverse ordering for negative entropies. No adaptive-order rescue claim is required by the present formulation.

\section{Deletion Envelopes and Constrained Selection}

\subsection{From a pooled measure to individual attention rows}

Write $a$ for a proposal attention row and $\omega$ for its positive mixture coefficient in~\eqref{eq:mixture}. For a removal set $E$, let $q=\mu^l(E)$ and $\eta=a(E)$. The relevant divergence is
\begin{equation}
 D_{\chi^2}(a\Vert\mu^l)=\sum_{j:\mu^l_j>0}\frac{(a_j-\mu^l_j)^2}{\mu^l_j}.
 \label{eq:chi}
\end{equation}
Mixture domination implies $a_j=0$ whenever $\mu^l_j=0$, so the expression is well-defined.

\begin{theorem}[Mixture-to-row deletion envelope]
\label{thm:envelope}
For every $E$ and every constituent proposal row,
\begin{equation}
 \eta\le\min\left\{1,\frac q\omega,
 q+\sqrt{D_{\chi^2}(a\Vert\mu^l)q(1-q)}\right\},
 \qquad D_{\chi^2}(a\Vert\mu^l)\le\omega^{-1}-1.
 \label{eq:envelope}
\end{equation}
Consequently, $\max_{1\le h\le H_Q}\max_{u\in\mathcal Q_p}A^{l,h}_u(E)
\le\min\{1,H_QW_pq/\zeta\}$. For head-averaged proposal rows $a^{l,h}$, the corresponding bound is $\min\{1,H_Qq/\zeta\}$.
\end{theorem}
\begin{proof}
The pointwise inequality $\mu^l_j\ge\omega a_j$ gives $\eta\le q/\omega$. On the support of $\mu^l$, put $f_j=a_j/\mu^l_j-1$. Then $\E_{\mu^l}f=0$ and
\[
 \eta-q=\E_{\mu^l}[f(\ind_E-q)].
\]
Cauchy--Schwarz bounds its absolute value by $\sqrt{\E_{\mu^l}f^2\,q(1-q)}$, yielding the third term. Also,
\[
 D_{\chi^2}(a\Vert\mu^l)=\sum_j\frac{a_j^2}{\mu^l_j}-1
 \le\omega^{-1}\sum_j a_j-1=\omega^{-1}-1.
\]
Substitute the row coefficient lower bound. Averaging observers first gives constituent coefficient $w_h/G\ge\zeta/H_Q$, proving the last statement.
\end{proof}

The envelope can be sharp. Let $a$ put all its mass on one position, let another distribution be supported off that position, and let $\mu=\omega a+(1-\omega)b$. For $E=\supp(a)$, $q=\omega$, $\eta=1$, and $D_{\chi^2}=\omega^{-1}-1$; both nontrivial envelope terms equal one. Thus a small pooled mass does not uniformly imply a small head-specific loss when the corresponding mixture weight is small. The factors $H_Q$ and $W_p$ expose a real coverage cost, rather than an artifact to omit from the guarantee.

For a fixed query and value vectors, the envelope yields a direct output statement.
\begin{lemma}[Sharp fixed-row deletion bound]
\label{lem:local}
Let $\eta=a(E)<1$, $\norm{v_j}\le V$, and
\begin{equation}
 o=\sum_j a_jv_j,\qquad
 o_R=\sum_{j\notin E}\frac{a_j}{1-\eta}v_j.
 \label{eq:row-output}
\end{equation}
Then $\norm{o-o_R}\le2V\eta$. The factor two is attainable on the probability simplex.
\end{lemma}
\begin{proof}
For $\eta>0$, write $o=(1-\eta)o_R+\eta o_E$. Both conditional averages have norm at most $V$, so $\norm{o-o_R}=\eta\norm{o_E-o_R}\le2V\eta$. The case $\eta=0$ is immediate. Assign every discarded value $Ve$ and every retained value $-Ve$ for a unit vector $e$ to attain equality. This proof requires fixed queries, retained keys, and values; it is not a theorem about later layers.
\end{proof}

\subsection{Budget feasibility and an optimality-gap certificate}

Index proposal constraints by $a=(h,u)$ and let $m_{ai}=\sum_{j\in c_i}A^{l,h}_{u,j}$. A removal decision $x_i\in\{0,1\}$ has token saving $\sum_i n_ix_i$. For row tolerances $\epsilon_a\in[0,1)$ and a minimum retained chunk count $m_0\le k$, the exact screening problem is
\begin{equation}
 \begin{aligned}
 D_l^*=\max_{x\in\{0,1\}^k}\;&\sum_i n_ix_i\\
 \text{subject to }&\sum_i m_{ai}x_i\le\epsilon_a\quad\text{for every }a,\\
 &\sum_i x_i\le k-m_0.
 \end{aligned}
 \label{eq:packing}
\end{equation}
Here $m_0\in\{0,\ldots,k\}$, so the empty removal is always feasible. A document-token target $B_{\rm pref}\in[0,N]$ requires a saving of at least $N-B_{\rm pref}$. Cardinality and token budgets are distinct, and retaining $m_0$ chunks does not establish retention of a reasoning chain.

The reference proposer sorts by increasing $S_i$, breaks ties by decreasing $n_i$ and then by original index, and scans once. It adds a chunk to $E$ only if every row constraint and the count constraint remain satisfied. It stops when the token-saving target is reached or no candidate remains. Its sorting cost is $O(k\log k)$ and its constraint-update cost is $O(kH_QW_p)$ once the chunk-mass table is available. This is a feasible heuristic, not an optimizer for~\eqref{eq:packing}.

\begin{theorem}[Dual certificate for feasible token removal]
\label{thm:dual}
For arbitrary nonnegative multipliers $y_a$ and $t$, define
\begin{equation}
 U_l(y,t)=\sum_a y_a\epsilon_a+t(k-m_0)
 +\sum_i\pos{n_i-\sum_a y_am_{ai}-t}.
 \label{eq:dual}
\end{equation}
Every feasible removal with saving $D_l$ satisfies
\begin{equation}
 D_l\le D_l^*\le U_l(y,t).
 \label{eq:dual-gap}
\end{equation}
If $U_l(y,t)<N-B_{\rm pref}$, the target is infeasible under the stated screening constraints. Otherwise the upper bound alone does not establish feasibility.
\end{theorem}
\begin{proof}
For a feasible $x$, subtract and add the constraint prices:
\[
 \sum_i n_ix_i
 =\sum_a y_a\sum_i m_{ai}x_i+t\sum_i x_i
 +\sum_i(n_i-\sum_a y_am_{ai}-t)x_i.
\]
The first two terms are bounded by their priced capacities. Since $0\le x_i\le1$, each final summand is at most the positive part of its coefficient. Taking the maximum over feasible binary $x$ proves the result.
\end{proof}

The certificate requires no assumption of strong duality or integrality. Prices can be chosen directly or obtained from an optional relaxation solver, whose cost must be charged to the implementation. The difference $U_l-D_l$ is a computable upper bound on the proposal's suboptimality in removed tokens. Failure of the greedy proposal to meet a target is not itself a proof that the target is impossible.

\section{Independent Auditing and Adaptive Layer Selection}
\label{sec:audit}

Let $\mathcal L\subseteq\{1,\ldots,L-1\}$ be a nonempty, prespecified monitored-layer set with $J=|\mathcal L|$. For each layer, define its candidate $E_l$ using the unpruned forward trajectory and proposal information only. These counterfactual candidates need not all be executed after a decision is accepted, but their definitions may not depend on audit samples. In particular, audit results may stop execution but may not be used to change candidates, thresholds, or the observer distribution within the same guarantee.

Let $\nu$ be a fixed distribution on user-query positions, for example uniform on the complete query. Fix $\delta\in(0,1)$, $\epsilon_{\rm audit}\in[0,1)$, and an integer $m_a\ge1$. Draw $m_a$ audit observers independently with replacement from $\nu$, independently of the proposal mechanism. At layer $l$, define
\begin{equation}
 \eta_{l,h}=\E_{u\sim\nu}[A^{l,h}_u(E_l)],\qquad
 \widehat\eta_{l,h}=\frac1{m_a}\sum_{s=1}^{m_a}A^{l,h}_{U_s}(E_l),
 \quad b=\sqrt{\frac{\ln(JH_Q/\delta)}{2m_a}}.
 \label{eq:audit}
\end{equation}
The same audit positions may be reused across layers; independence between different layer--head estimates is not required. The samples within an estimate are independent draws from $\nu$ conditional on the prompt, model, and proposal randomness.

\begin{theorem}[Simultaneous observer guarantee]
\label{thm:audit}
With probability at least $1-\delta$ over the audit sampling, every monitored layer and head satisfies
\begin{equation}
 \eta_{l,h}\le\min\{1,\widehat\eta_{l,h}+b\}.
 \label{eq:audit-bound}
\end{equation}
Therefore any layer selected after examining these bounds has $\eta_{l^*,h}\le\epsilon_{\rm audit}$ for every head whenever its acceptance rule requires $\widehat\eta_{l^*,h}+b\le\epsilon_{\rm audit}$.
\end{theorem}
\begin{proof}
Conditional on the prompt, model, and proposal mechanism, each audit observation is in $[0,1]$ with mean $\eta_{l,h}$. The one-sided Hoeffding inequality gives
\[
 \Prb(\eta_{l,h}>\widehat\eta_{l,h}+b)\le e^{-2m_ab^2}
 =\frac{\delta}{JH_Q}.
\]
A union bound over the $JH_Q$ pairs establishes a simultaneous event \citep{hoeffding1963}. Selection of a layer within that event cannot invalidate its already established inequality. Averaging the conditional probability statement over proposal randomness proves the unconditional claim.
\end{proof}

This controls expected discarded mass under $\nu$, not the worst observer, every document position, or future decode attention. A deterministic last-token observation does not satisfy the sampling theorem merely because it is called an audit. If $b$ exceeds the configured tolerance, no layer passes; one must enlarge the audit, compute the finite observer population exactly, relax the tolerance with an explicit scope change, or bypass pruning. Using audit outcomes to retune the proposal requires a fresh independent audit or a different uniform-convergence argument.

For a valid candidate, define retained chunks $\mathcal R_l$ and the set-change statistic
\begin{equation}
 d_l=\frac{|\mathcal R_l\triangle\mathcal R_{l^-}|}{\max\{1,|\mathcal R_l\cup\mathcal R_{l^-}|\}},
 \label{eq:stability}
\end{equation}
where $l^-$ is the preceding monitored layer. A streak of $v$ valid transitions with $d_l\le\theta$ is an optional scheduling condition; it is not a correctness theorem. The first valid candidate initializes the predecessor. An invalid proposal resets the streak. The method accepts the earliest layer meeting the token target, the proposal constraints, the stability condition, and the audit threshold. If none passes, the complete prompt continues.

\Needspace{14\baselineskip}
\noindent\textbf{Algorithm 1: EntroPrefill with independent observer audit}
\begin{enumerate}[leftmargin=1.5em,itemsep=1pt,topsep=4pt,label=\arabic*.]
\item Fix protected support, monitored layers, all thresholds, proposal observers, audit sample size, and audit distribution before observing audit values.
\item Execute unpruned blocks sequentially; at each monitored layer, extract proposal attention and compute~\eqref{eq:collision}--\eqref{eq:score}.
\item If residual or document mass is invalid, reset the predecessor and streak and continue without pruning.
\item Construct $E_l$ with the feasible scan for~\eqref{eq:packing}; update the candidate-stability streak independently of audit values.
\item If the saving target or streak requirement fails, continue; otherwise evaluate the fixed independent audit and~\eqref{eq:audit-bound}.
\item If every head passes, set $l^*=l$, gather hidden states and shallow KV on $R=[T]\setminus E_l$, preserve original positions, and process all remaining blocks on $R$.
\item If the monitored window ends without acceptance, finish ordinary full-context prefill. Decode using the resulting cache and the configured page policy.
\end{enumerate}

\section{Conditional Transformer Output Stability}
\label{sec:network}

\subsection{Why a shallow certificate is insufficient}

\begin{proposition}[No uniform output bound from shallow attention]
\label{prop:impossible}
Consider attention networks that can preserve a token-specific value coordinate through a shallow prefix and use an unconstrained later query--key map. There exist networks and prompt pairs with identical shallow attention and identical retained states for which the shallow attention mass on a removed position tends to zero, while at least one future-output deletion error remains bounded away from zero. Therefore no bound depending only on shallow discarded attention can vanish uniformly over this class.
\end{proposition}
\begin{proof}
Take two prompts differing only in a value coordinate at a position $j^*$, and remove $E=\{j^*\}$. Set every shallow attention-output projection and feed-forward map to zero, so the shallow prefix is the identity. Choose shallow query and key maps that ignore the value coordinate and give mass $\varepsilon>0$ to $j^*$. The retained states and shallow scores agree between prompts, whereas the removed values are $+V$ and $-V$ for $V>0$. A later head uses a separate position-specific key coordinate to give reference mass $1-\gamma$ to $j^*$; its value map reads only the signed coordinate. Set other values to zero. The two reference outputs are $+(1-\gamma)V$ and $-(1-\gamma)V$, whereas reduced outputs agree because their inputs agree. The triangle inequality makes at least one error at least $(1-\gamma)V$. Taking $\varepsilon\to0$ with fixed $\gamma<1$ proves the claim; taking $\gamma\to0$ makes the lower bound approach $V$. A two-logit readout can make the reference first-token decisions differ. This construction concerns attention-based evidence for deleting $E$; it does not assert that every conceivable pruning rule chooses $E$.
\end{proof}

The proposition does not claim impossibility for a fixed model whose later weights and behavior are additionally constrained. It identifies the missing premise in an unconditional inference from shallow concentration. The following theorem states sufficient premises for that restricted setting.

\subsection{Block constants and deletion defects}

Consider pre-normalized blocks
\begin{equation}
 B_l(X)=X+O_l\operatorname{Attn}_l(N_l(X)),\qquad
 F_l(X)=B_l(X)+\operatorname{MLP}_l(\overline N_l(B_l(X))).
 \label{eq:block}
\end{equation}
Heads are concatenated before $O_l$. Let $c_l$ and $\bar c_l$ be rowwise Lipschitz constants for the two normalizers. Let $f_l$ bound the MLP Lipschitz constant on a convex enclosure of its normalized input domain, and put $R_l=1+f_l\bar c_l$. Assume normalized-domain query, key, and value norms are bounded by $\overline Q_{l,h}$, $\overline K_{l,h}$, and $V_{l,h}$. The matrices $W^Q_{l,h}$, $W^K_{l,h}$, and $W^V_{l,h}$ include the appropriate shared group projections. Biases are absorbed into the norm bounds. Fixed RoPE rotations preserve these norms.

Define sufficient constants
\begin{align}
 a_{l,h}&=\frac{c_l}{\sqrt d}
 \left(\overline Q_{l,h}\norm{W^K_{l,h}}+
 \overline K_{l,h}\norm{W^Q_{l,h}}\right),\label{eq:a-const}\\
 b_{l,h}&=c_l\norm{W^V_{l,h}}+2V_{l,h}a_{l,h},\qquad
 K_l=R_l\left(1+\norm{O_l}\sqrt{\sum_h b_{l,h}^2}\right).
 \label{eq:k-const}
\end{align}
For a fixed removal $E$ and retained set $R$, let $X^l$ be the full reference trajectory and
\begin{equation}
 \eta_{l,h}^{\rm ref}=\max_{u\in R}\sum_{j\in E}A^{l,h}_{u,j}(X^{l-1}),
 \qquad e_l=2R_l\norm{O_l}\sqrt{\sum_h V_{l,h}^2(\eta_{l,h}^{\rm ref})^2}.
 \label{eq:local-defect}
\end{equation}
The maximum includes \emph{every} retained query row, not only the user-query observers. Protected causal positions guarantee nonempty retained attention support.

\begin{theorem}[Conditional first-token perturbation]
\label{thm:network}
Under the stated domain bounds, let the reduced trajectory start at $Y^{l^*}=X^{l^*}|_R$ and apply the same blocks on support $R$ with original position IDs. Then
\begin{equation}
 \Delta_L:=\norm{X^L|_R-Y^L}_{\max,2}
 \le\sum_{j=l^*+1}^{L}e_j\prod_{s=j+1}^{L}K_s.
 \label{eq:propagation}
\end{equation}
If the final query-row readout is $\Lambda$-Lipschitz into the logit $\ell_\infty$ norm, then
\begin{equation}
 \norm{z-\widetilde z}_\infty\le\Lambda\Delta_L.
 \label{eq:logit-bound}
\end{equation}
\end{theorem}
\begin{proof}
Appendix~\ref{app:constants} proves that the retained-support block $F_{l,R}$ is $K_l$-Lipschitz. Apply the full and restricted block to the same reference retained states. For each retained query, its query and surviving keys and values agree; restricting softmax to $R$ renormalizes the surviving probabilities. Lemma~\ref{lem:local} bounds the difference for head $h$ by $2V_{l,h}\eta_{l,h}^{\rm ref}$. Concatenation, output projection, and the tokenwise post-attention map give
\[
 \norm{F_l(X^{l-1})|_R-F_{l,R}(X^{l-1}|_R)}_{\max,2}\le e_l.
\]
Add and subtract $F_{l,R}(X^{l-1}|_R)$ to obtain $\Delta_l\le e_l+K_l\Delta_{l-1}$. The initial discrepancy is zero. Unrolling the recurrence proves~\eqref{eq:propagation}; the readout Lipschitz inequality proves~\eqref{eq:logit-bound}.
\end{proof}

\begin{corollary}[Decision margin]
\label{cor:margin}
Let the full-context first-token winner have logit margin
$\mathfrak m=z_{(1)}-z_{(2)}>0$. If $\mathfrak m>2\Lambda\Delta_L$, the reduced and full computations have the same greedy first token.
\end{corollary}
\begin{proof}
The winning logit decreases by at most $\Lambda\Delta_L$, and any competing logit increases by at most that amount. The remaining margin is strictly positive.
\end{proof}

This is a first-token statement. Extending it to a generated sequence requires new bounds on every shared generation prefix and sufficient margins at every step. The theorem supplies neither those future margins nor equality of sampled generations.

\subsection{A sufficient bridge from the probe to deeper layers}

A precise support-transfer assumption can connect the probe measure to~\eqref{eq:local-defect}. Suppose that for each deeper layer and head, every retained reference row satisfies
\begin{equation}
 A^{l,h}_u(D')\le\Gamma_{l,h}\mu^{l^*}(D')+\sigma_{l,h}
 \quad\text{for every }D'\subseteq\mathcal D,
 \label{eq:domination}
\end{equation}
where $\Gamma_{l,h}\ge0$ and $\sigma_{l,h}\ge0$ are known bounds. Then
\begin{equation}
 \eta_{l,h}^{\rm ref}\le\min\{1,\Gamma_{l,h}q+\sigma_{l,h}\},
 \qquad q=\mu^{l^*}(E).
 \label{eq:bridge}
\end{equation}
Substitution into~\eqref{eq:local-defect} and Theorem~\ref{thm:network} gives a probe-dependent output bound. A sufficient construction is a density ratio $A^{l,h}_{u,j}/\mu^{l^*}_j\le\Gamma_{l,h}$ on documents, with $\sigma_{l,h}=0$. More generally, the total positive excess $\sum_{j\in\mathcal D}\pos{A^{l,h}_{u,j}-\Gamma_{l,h}\mu^{l^*}_j}$ can be bounded by $\sigma_{l,h}$.

Equation~\eqref{eq:domination} is an assumption about later reference attention, not an output of the shallow audit. It can be loose or fail to provide a useful bound: tiny probe probabilities can require large density ratios, and products of $K_l$ can amplify small defects. The implemented acceptance rule certifies its stated observer quantities; an end-to-end certificate additionally requires justified bounds of this stronger kind. In particular, merely measuring stable chunk ranks does not establish~\eqref{eq:domination}.

\section{GQA Unions and Physical Page Budgets}

If each sibling head independently selects $s$ token positions and the group retains every selection, then
\begin{equation}
 \UOR_g=\frac{|\bigcup_{h\in\mathcal H_g}\mathcal K_h|}{s},
 \qquad 1\le\UOR_g\le\min\{r,T/s\}.
 \label{eq:uor}
\end{equation}
For two equal-cardinality sets with Jaccard disagreement $D$, inclusion--exclusion gives $\UOR=2/(2-D)$. These are set identities, not claims about all implementations of an existing cache policy. A policy that already shares one selection within each physical group does not incur this particular intra-group union.

For decode integration, let $p_{l,g}$ be a normalized group importance distribution on resident tokens and let page capacity be $b$. Page-EntroKV motivates the page priority
\begin{equation}
 I_{l,g}(P)=\max_{j\in P}p_{l,g}(j).
 \label{eq:page-max}
\end{equation}
Protect sink and recent pages, including the writable page, and fill a specified total quota with top-scoring remaining pages. Exact quota cardinality requires enough available pages and a quota at least as large as the protected set. Selecting one common page set for all sibling query heads makes the intra-group UOR one by construction. It does not prove that block-max is statistically preferable to page mass for all tasks.

\begin{proposition}[Competitive page retention and allocation]
\label{prop:pages}
If $K\ge1$ nonprotected pages are selected by block-max and a token has pooled mass greater than $1/(K+1)$, its page is retained. Separately, if $G$ groups each choose $K$ logical pages from $N_P$ candidates, their union satisfies
\begin{equation}
 K\le U_l:=|\bigcup_g\mathcal K_{l,g}|\le\min\{GK,N_P\}.
 \label{eq:page-union}
\end{equation}
\end{proposition}
\begin{proof}
If the token's page were omitted, $K$ distinct competing pages would have maxima at least as large. Those pages and its own page would contribute more than $(K+1)/(K+1)=1$ total probability, a contradiction. For the union, any one selected set supplies the lower bound, while subadditivity and the finite candidate pool supply the upper bounds.
\end{proof}

If pages are independently allocated per group, their KV payload is $2edb\sum_{l,g}K_{l,g}$ bytes for $e$ bytes per element. If a physical layer block contains all $G$ groups, payload under union materialization is instead $2edbG\sum_l U_l$. If allocation also spans layers, its actual release unit must replace this layer-local accounting. Neither formula includes metadata, padding outside the nominal block, quantization scales, allocator reservation, or transport headers. Whole-page selection avoids token-level holes introduced by partial deletion, but it does not eliminate tail slack or irrelevant tokens co-resident with a retained fact. Group-level UOR equal to one is therefore not a zero-fragmentation theorem.

Decode weights must be initialized for each managed layer; shallow-layer weights cannot stand in for the entire cache. Recomputing scores from a compressed cache accesses resident tokens only. Recovering evicted tokens requires an explicit offload or recomputation mechanism. Prefill pruning and decode eviction can share scoring concepts, but deleting positive attention mass and renormalizing it is not measure preservation.

\section{Computational and Serving Consequences}

\subsection{Prefill break-even and score-extraction cost}

Let the dominant full-layer floating-point cost be $F_l(T)=a_lT+b_lT^2$, with nonnegative coefficients and the linear term including projections and feed-forward computation. For accepted pruning,
\begin{equation}
 F_{\rm EP}=\sum_{l\le l^*}F_l(T)+\sum_{l>l^*}F_l(T')+F_{\rm aux}.
 \label{eq:flops}
\end{equation}
Arithmetic saving is positive precisely when the removed deep-layer arithmetic exceeds $F_{\rm aux}$. For identical layers and $\xi=T'/T$, the ideal saving before auxiliary work is
\begin{equation}
 \left(1-\frac{l^*}{L}\right)
 \frac{aT(1-\xi)+bT^2(1-\xi^2)}{aT+bT^2}.
 \label{eq:savings}
\end{equation}
For positive total cost it lies between $(1-l^*/L)(1-\xi)$ and $(1-l^*/L)(1-\xi^2)$. With $L=32$, $l^*=4$, and $\xi=1/2$, these endpoints are $43.75\%$ and $65.625\%$. They are analytic examples, not measured reductions. A nonzero dense shallow window preserves quadratic asymptotic attention complexity in $T$.

Exact proposal-row extraction can require $O(H_QW_pTd)$ dot-product work per monitored layer. An explicit reference path stores the averaged head distributions in $O(H_QT)$ space and the row--chunk mass table in $O(H_QW_pk)$ space. Entropy weights are available only after the head statistics are formed, so exact pooling generally needs a subsequent pass. Audit rows add their own capture or recomputation cost. A full $T\times T$ attention matrix is unnecessary, but zero intermediate HBM traffic is not implied.

For one attention row, stable online accumulators satisfy
\begin{align}
 \ell'&=e^{m-m'}\ell+\sum_{j\in\mathrm{tile}}e^{z_j-m'},\nonumber\\
 s'&=e^{2(m-m')}s+\sum_{j\in\mathrm{tile}}e^{2(z_j-m')},
 \qquad \sum_j A_j^2=s/\ell^2,
 \label{eq:online}
\end{align}
where $m'$ is the updated maximum. This computes collision for one row. It does not compute collision after observer or head averaging, since averaging before squaring introduces cross-products. A fused kernel must implement the specified pooled statistic or explicitly introduce a separately analyzed approximation. Register pressure, shared-memory use, L2 displacement, kernel launches, and synchronization require measurement.

Compaction gathers retained states and shallow KV entries and preserves token order and original position IDs. The next generated position follows the original prompt range, not the compact storage length. Pure gathers contribute byte traffic rather than floating-point arithmetic; temporary allocations and copies must enter latency and peak-memory accounting. Retained states still encode earlier interaction with the full prompt, so this path is not equivalent to textual truncation followed by a fresh forward pass.

\subsection{Disaggregation and queueing}

DistServe motivates distinct prefill and decode resource pools, while Sarathi-Serve supplies a relevant colocated scheduling alternative \citep{zhong2024,agrawal2024}. With a dense retained cache and no replication or quantization metadata, transfer payload is $2eLGdT'$ bytes. A page-based implementation must use allocated slots and its actual group/layer union instead. If decode selection occurs only after transfer, it does not reduce the bytes already sent.

Writing actual wire payload as $B_{\rm wire}$ and achieved bandwidth as $\beta_{\rm eff}$ gives $t_{\rm xfer}\ge B_{\rm wire}/\beta_{\rm eff}$. Queueing, layout conversion, and setup are additional costs. Emitting the first token on the prefill worker can move transfer delay into the first-to-second-token interval rather than removing it. Different requests also prune at different layers, requiring a defined batch-cohort or device-routing policy; padding retained lengths alone does not resolve depth divergence. Tensor-parallel grouping can introduce score reductions and index broadcasts. Thus a compute reduction, an allocation reduction, and a goodput improvement are separate propositions.

\section{Scope and Planned Empirical Validation}

This manuscript establishes a theoretical procedure and conditional guarantees. It reports no task benchmark, kernel speedup, throughput increase, or empirical noninferiority result. Theoretical examples and implementation-level algebra checks must not be presented as model experiments. The next empirical phase should determine whether entropy priorities obtain better feasible removals than mass-only or uniform priorities, whether the audit is nonvacuous at practical sample sizes, and whether the conditional stability constants explain observed perturbations.

Evaluation should compare full-context inference, prompt compression, fixed-layer pruning, LazyLLM, SlimInfer, ASL, and implementation-appropriate cache baselines under matched token and physical-byte budgets. Single-hop retrieval, multi-hop bridges, diffuse evidence, contradictory answer strings, and reordered documents should be analyzed separately. Key diagnostics are proposal saving $D_l$, dual gap $U_l-D_l$, audit coverage, actual future discarded masses, density-ratio excess, first-token logit margins, full support-chain retention, and bypass frequency. A held-out calibration protocol must fix all thresholds before the reported test evaluation; an input-level observer confidence statement is not a guarantee across future requests or model families.

Accuracy preservation should use a prespecified noninferiority margin and paired uncertainty estimates. Systems measurements should include synchronized TTFT, first-to-second-token delay, decode latency, request completion time, peak allocated and reserved memory, transfer bytes, and SLO goodput at equal total GPU budgets. Weight precision, checkpoint revision, workload distribution, compilation, warmup, and arrival process must be fixed or disclosed. An integrated EntroPrefill--Page-EntroKV--disaggregation comparison needs component ablations rather than a single combined speedup.

The primary limitations are structural. Concentrated attention can select a false fact. Bridge relevance can emerge only in deeper layers. An independent audit controls its observer distribution and can be too conservative for short queries. The full-model bound can be vacuous because of large Lipschitz products or poor support transfer. Irreversible deletion limits recovery during long generation and multi-turn reuse. Finally, kernel compatibility, page release semantics, and distributed scheduling remain implementation obligations. These limitations identify the conditions under which the proposed theory can become useful, rather than being exceptions hidden behind a universal preservation claim.

\section{Conclusion}

EntroPrefill couples a sink-isolated R\'enyi proposal with explicit deletion constraints and independent adaptive-layer auditing. The analysis characterizes the coverage lost through head pooling, bounds attainable token removal, and states sufficient conditions for first-token output stability. A complementary counterexample prevents shallow attention from being interpreted as an unconditional certificate. Physical allocation and serving costs remain explicit in the composition with page eviction and disaggregation. Establishing practical advantage requires implementing this exact procedure and testing whether its guarantees are informative at useful pruning budgets.

\appendix
\section{Proof of the Transformer Block Constants}
\label{app:constants}

\subsection{Normalizer and softmax bounds}

For regularized LayerNorm in dimension $d_m$, let $P=I-\mathbf1\mathbf1^\top/d_m$, $z=Px$, and $v=\norm{z}^2/d_m+\epsilon_{\rm LN}$, with $\epsilon_{\rm LN}>0$. With learned scale $\gamma$ and bias $\beta$,
\begin{equation}
 N(x)=\gamma\odot\frac{z}{\sqrt v}+\beta,\qquad
 J_N(x)=\diag(\gamma)\left(\frac P{\sqrt v}-\frac{zz^\top}{d_m v^{3/2}}\right).
 \label{eq:ln-jac}
\end{equation}
On the mean-zero subspace, directions orthogonal to $z$ have eigenvalue $1/\sqrt v$ before scaling, while the direction of $z$ has eigenvalue $\epsilon_{\rm LN}/v^{3/2}$. The constant direction has eigenvalue zero. Consequently,
\begin{equation}
 \norm{J_N(x)}\le\frac{\norm{\gamma}_\infty}{\sqrt{\epsilon_{\rm LN}}},
 \qquad \norm{N(x)}\le\norm{\gamma}_\infty\sqrt{d_m}+\norm{\beta}.
 \label{eq:ln-constants}
\end{equation}
The first is a global Lipschitz bound; the second supplies a bounded normalized domain for query, key, and value projections. RMS normalization admits the corresponding argument with $P=I$. The constants can be very large at small numerical regularization and should not be confused with tight empirical estimates.

For $p=\softmax(z)$ and perturbation direction $v$, its Jacobian gives $(Jv)_i=p_i(v_i-\sum_jp_jv_j)$. Therefore
\begin{equation}
 \norm{Jv}_1\le\sum_i p_i\left(|v_i|+\left|\sum_jp_jv_j\right|\right)
 \le2\norm{v}_\infty.
 \label{eq:softmax-lip}
\end{equation}
Integration along the segment between logits proves $\norm{\softmax(z)-\softmax(z')}_1\le2\norm{z-z'}_\infty$. This deliberately conservative cross-norm constant suffices here; no optimality claim is made for it. For causal attention, the calculation is performed on each row's fixed allowed support.

\subsection{Attention and feed-forward composition}

Let two retained-state matrices differ by at most $\Delta$ in row-maximum norm. Their normalized states differ by at most $c_l\Delta$. By adding and subtracting a mixed query--key product, every head logit changes by at most
\[
 \frac{\overline Q_{l,h}\norm{W^K_{l,h}}+
 \overline K_{l,h}\norm{W^Q_{l,h}}}{\sqrt d}\,c_l\Delta
 =a_{l,h}\Delta.
\]
The corresponding value difference is at most $c_l\norm{W^V_{l,h}}\Delta$. Decompose the output difference into a value-change term and a probability-change term. Equation~\eqref{eq:softmax-lip} bounds the latter by $2V_{l,h}a_{l,h}\Delta$, so the total is at most $b_{l,h}\Delta$. Concatenation across heads changes the row norm by at most $\sqrt{\sum_hb_{l,h}^2}\Delta$. Multiplication by $O_l$ and addition of the attention residual produce the factor $1+\norm{O_l}\sqrt{\sum_hb_{l,h}^2}$.

The tokenwise map $x\mapsto x+\operatorname{MLP}_l(\overline N_l(x))$ has Lipschitz constant at most $R_l=1+f_l\bar c_l$, proving~\eqref{eq:k-const}. For a two-matrix MLP $W_2\phi(W_1x+b_1)+b_2$, one may take $f_l\le\norm{W_2}L_\phi\norm{W_1}$ when $\phi$ is $L_\phi$-Lipschitz. For gated MLPs, a bound on their actual Jacobian over a convex bounded enclosure of the normalized domain is required instead; the two-matrix expression cannot be substituted unchanged. The final readout constant can be taken as $\Lambda=\norm{W_{\rm out}}_{2\to\infty}c_{\rm final}$ when final normalization has Lipschitz constant $c_{\rm final}$.

\section{Limits of Entropy-Only and Cross-Phase Guarantees}

The mixture envelope depends on both head coverage and the distributional discrepancy. A low collision entropy alone controls neither the value directions nor the truth of the attended statement. The sharp construction in Lemma~\ref{lem:local} shows that value information is necessary to improve a worst-case $2V\eta$ guarantee. Conversely, equal values can make deletion harmless despite a large discarded mass. Thus the attention-only bound is conservative for some prompts and cannot be universally tightened using entropy without further assumptions.

Likewise, a common R\'enyi functional does not imply conservation across page partitions. Consider one chunk made of two equal pages of size $b$, with masses $1-\alpha$ and $\alpha$, uniform within pages. Its negative collision entropy is
\begin{equation}
 \Phi_2(c)=-\ln b+\ln((1-\alpha)^2+\alpha^2).
 \label{eq:partition-chunk}
\end{equation}
If a page score is its conditional negative collision entropy plus log page mass, the arithmetic mean of the two page scores is
\begin{equation}
 \frac{\Psi(P_1)+\Psi(P_2)}2=-\ln b+\frac12\ln(\alpha(1-\alpha)).
 \label{eq:partition-page}
\end{equation}
The difference diverges as $\alpha\to0$, despite perfect page alignment and identical distributions in the two phases. A boundary-only error term therefore cannot justify such a conservation identity. The admissible cross-phase statements in this paper are explicit resource counts and conditional output inequalities, not measure preservation.

\section{Implementation Invariants and Reproducibility Contract}

An implementation must reproduce observer averaging before entropy, the regularization floor, complete-chunk selection, deterministic ties, and the independence condition of Theorem~\ref{thm:audit}. Row-wise collision accumulation followed by averaging is a different operator. The audit must not silently reuse proposal-dependent sample selection. Threshold changes triggered by audit failures require a separately justified protocol. Empty candidate sets, all-ineligible groups, nonfinite statistics, insufficient protected-page capacity, and failed token targets must have explicit fallbacks.

Retained positions must stay in original order, protected positions must survive, and logical sequence positions must remain distinct from compact storage offsets. Query, key, value, normalization, and readout constants in Theorem~\ref{thm:network} must correspond to the same checkpoint and numerical execution assumptions. A floating-point implementation may require additional rounding-error terms if a formal numerical certificate is claimed. Analytic operator bounds in real arithmetic are not automatically machine-verified guarantees.

The future artifact should release the exact scoring reference, optimized-kernel comparisons against it, the candidate and audit observer indices, constraint tables, dual prices when used, retained-index maps, and per-request resource traces. Experiments should retain failed and bypassed cases in the denominator. This contract makes the theoretical objects reproducible without treating implementation completion or empirical validation as already established.


\begin{thebibliography}{99}
\small
\setlength{\itemsep}{3pt}
\bibitem{jiang2023} H. Jiang et al. LLMLingua: Compressing prompts for accelerated inference of large language models. EMNLP, 2023. \url{https://arxiv.org/abs/2310.05736}.
\bibitem{fang2025} Y. Fang et al. AttentionRAG: Attention-guided context pruning in retrieval-augmented generation. 2025. \url{https://arxiv.org/abs/2503.10720}.
\bibitem{fu2024} Q. Fu et al. LazyLLM: Dynamic token pruning for efficient long context LLM inference. 2024. \url{https://arxiv.org/abs/2407.14057}.
\bibitem{long2025} L. Long et al. SlimInfer: Accelerating long-context LLM inference via dynamic token pruning. 2025. \url{https://arxiv.org/abs/2508.06447}.
\bibitem{taniguchi2026} R. Taniguchi et al. Adaptive layer selection for layer-wise token pruning in LLM inference. Findings of ACL, 2026. \url{https://arxiv.org/abs/2601.07667}.
\bibitem{zhang2023} Z. Zhang et al. H2O: Heavy-hitter oracle for efficient generative inference of large language models. 2023. \url{https://arxiv.org/abs/2306.14048}.
\bibitem{li2024} Y. Li et al. SnapKV: LLM knows what you are looking for before generation. 2024. \url{https://arxiv.org/abs/2404.14469}.
\bibitem{feng2024} Y. Feng et al. Ada-KV: Optimizing KV cache eviction by adaptive budget allocation for efficient LLM inference. 2024; revised 2025. \url{https://arxiv.org/abs/2407.11550}.
\bibitem{ainslie2023} J. Ainslie et al. GQA: Training generalized multi-query transformer models from multi-head checkpoints. EMNLP, 2023. \url{https://arxiv.org/abs/2305.13245}.
\bibitem{pageentro} Inbasekaran S. Page-EntroKV: Hardware-aligned, entropy-weighted KV-cache eviction under grouped-query attention. Unpublished author manuscript supplied with this work, 2026.
\bibitem{xiao2024} G. Xiao et al. Efficient streaming language models with attention sinks. ICLR, 2024. \url{https://arxiv.org/abs/2309.17453}.
\bibitem{wu2024} W. Wu et al. Retrieval head mechanistically explains long-context factuality. 2024. \url{https://arxiv.org/abs/2404.15574}.
\bibitem{dao2023} T. Dao. FlashAttention-2: Faster attention with better parallelism and work partitioning. 2023. \url{https://arxiv.org/abs/2307.08691}.
\bibitem{kwon2023} W. Kwon et al. Efficient memory management for large language model serving with PagedAttention. SOSP, 2023. \url{https://arxiv.org/abs/2309.06180}.
\bibitem{zhong2024} Y. Zhong et al. DistServe: Disaggregating prefill and decoding for goodput-optimized large language model serving. 2024. \url{https://arxiv.org/abs/2401.09670}.
\bibitem{agrawal2024} A. Agrawal et al. Taming throughput-latency tradeoff in LLM inference with Sarathi-Serve. 2024. \url{https://arxiv.org/abs/2403.02310}.
\bibitem{vanerven2014} T. van Erven and P. Harremo\"es. R\'enyi divergence and Kullback--Leibler divergence. IEEE Transactions on Information Theory, 60(7):3797--3820, 2014. \url{https://arxiv.org/abs/1206.2459}.
\bibitem{hoeffding1963} W. Hoeffding. Probability inequalities for sums of bounded random variables. Journal of the American Statistical Association, 58(301):13--30, 1963. \url{https://doi.org/10.1080/01621459.1963.10500830}.
\end{thebibliography}
\end{document}